\documentclass{article}

\usepackage{PRIMEarxiv}

\usepackage{times}
\usepackage{soul}
\usepackage{url}
\usepackage[hidelinks]{hyperref}
\usepackage[utf8]{inputenc}
\usepackage[small]{caption}
\usepackage{graphicx}
\usepackage{amsmath}
\usepackage{amsthm}
\usepackage{booktabs}
\usepackage{algorithm}
\usepackage{algorithmic}
\usepackage[switch]{lineno}
\usepackage{amsfonts}
\usepackage{bbm}
\usepackage{subcaption}
\usepackage{multirow}
\usepackage{makecell}
\usepackage{amssymb, geometry}
\usepackage[table]{xcolor}
\usepackage{natbib}

\graphicspath{{media/}}     

\newtheorem{theorem}{Theorem}
\newtheorem{lemma}{Lemma}

\title{From Shallow to Deep: Pinning Semantic Intent via Causal GRPO}

\author{
  Shuyi Zhou$^{1,3}$, Zeen Song$^{1,2}$, Wenwen Qiang$^{2}$, Jiyan Sun$^{3}$, Yao Zhou$^{1,2}$, Yinlong Liu$^{3}$, Wei Ma$^{3}$ \\
  \vspace{0.3em}
  $^{1}$University of Chinese Academy of Sciences, Beijing, China \\
  $^{2}$Institute of Software, Chinese Academy of Sciences, Beijing, China \\
  $^{3}$Institute of Information Engineering, Chinese Academy of Sciences, Beijing, China \\
  \vspace{0.3em}
  \texttt{\{zhoushuyi,sunjiyan,liuyinlong,mawei\}@iie.ac.cn} \\
  \texttt{\{songzeen,zhouyao\}@iscas.ac.cn} \\
  \texttt{qiang.ww0922@gmail.com}
}

\begin{document}
\maketitle

\begin{abstract}
Large Language Models remain vulnerable to adversarial prefix attacks (e.g., ``Sure, here is'') despite robust standard safety. We diagnose this vulnerability as Shallow Safety Alignment, stemming from a pathology we term semantic representation decay: as the model generates compliant prefixes, its internal malicious intent signal fades. To address this, we propose Two-Stage Causal-GRPO (TSC-GRPO), a framework designed to achieve intent pinning. First, grounded in causal identifiability theory, we train a causal intent probe to disentangle invariant intent from stylistic perturbations. Second, we internalize this causal awareness into the policy via Group Relative Policy Optimization. By employing a cumulative causal penalty within ``fork-in-the-road'' training scenarios, we force the model to learn that accumulating harmful tokens monotonically decreases reward, enabling robust late-stage refusals. Experiments show that TSC-GRPO significantly outperforms baselines in defending against jailbreak attacks while preserving general utility.
\end{abstract}


\section{Introduction}
\label{sec:intro}

The safety alignment of Large Language Models (LLMs) has become a paramount concern as their deployment scales. Techniques like Supervised Fine-Tuning (SFT) and Reinforcement Learning from Human Feedback (RLHF) have successfully instilled a refusal behavior for explicit harmful queries (e.g., ``How to build a bomb?''). However, recent studies reveal that this alignment is often fragile and ``skin-deep'' \cite{qi2025safety,zou2023universal}. A simple adversarial maneuver—such as injecting a compliant prefix like ``Sure, here is''—can effectively bypass these defenses, causing the model to generate prohibited content.

Why do robustly aligned models fail so catastrophically against simple prefixes? Current prevailing methods largely treat safety as a behavioral optimization problem: they penalize the model for outputting harmful tokens. However, they fail to constrain the internal decision process that leads to these outputs. Through empirical analysis (detailed in Section \ref{sec:causal_analysis}), we diagnose the root cause as Semantic Representation Decay. We observe that while the model initially recognizes the malicious intent of a query, this internal signal is unstable. As the model auto-regressively generates a forced compliant prefix, the ``intent'' representation is overwritten by the ``style'' of compliance. The model literally ``loses sight'' of the harm, reducing safety alignment to a fragile game of ``Whac-A-Mole''—blocking specific keywords without fixing the underlying blindness.

To resolve this, we propose a paradigm shift from ``Behavioral Patching'' to ``Deep Causal Intervention.'' We propose the Two-Stage Causal-GRPO (TSC-GRPO), designed to achieve Intent Pinning: ensuring the model's internal awareness of harm remains invariant, regardless of the generated context or adversarial prefixes.

TSC-GRPO operates in two coupled stages. In Stage 1, we address the ``Visibility Problem.'' Since standard classifiers conflate intent (Content) with wording (Style), we leverage causal representation learning theory to train a Causal Intent Probe. By utilizing a novel hard-negative augmentation strategy, we force the probe to distinguish between ``Malicious Compliance'' (e.g., a bomb recipe starting with ``Sure'') and ``Safe Refusal,'' creating a ``Semantic Compass'' that points to danger even when obscured by adversarial styles. In Stage 2, we address the ``Inertia Problem.'' We employ Group Relative Policy Optimization (GRPO) to internalize this compass into the model's policy. Unlike standard RLHF which uses sparse rewards, we introduce a cumulative causal penalty. We construct ``Fork-in-the-Road'' scenarios where the model is initialized with a harmful prefix and forced to choose between continuing the harm or pivoting to a refusal. By penalizing the accumulation of harmful semantic signals token-by-token, the model learns a robust policy: maximize reward by breaking the semantic link to harm immediately, even if the sentence started with ``Sure.'' 

Our contributions are summarized as: 1) We identify Semantic Representation Decay as the mechanistic cause of shallow alignment failure, providing empirical evidence; 2) We propose TSC-GRPO, a theoretically grounded framework that combines Causal Disentanglement (Stage 1) with GRPO (Stage 2) to achieve Intent Pinning; 3) Extensive experiments verify that TSC-GRPO enhances robustness against jailbreak attacks, without compromising general model capabilities.

\section{Related Work}
LLMs have achieved remarkable success across diverse domains, demonstrating unprecedented capabilities in natural language understanding, reasoning, and code generation \cite{openaiGPT4TechnicalReport2023,deepseek2025r1}. However, these powerful capabilities simultaneously endow LLMs with the potential to generate harmful content, such as hate speech \cite{gehman2020realtoxicityprompts,weidinger2021ethical}, instructions for illicit activities \cite{brundage2018malicious}, and disinformation \cite{vykopal2024disinformation}. The dual-use nature of these models poses significant societal risks if left unchecked \cite{weidinger2021ethical,brundage2018malicious}. Consequently, implementing robust safety alignment mechanisms is indispensable to prevent the generation of malicious outputs and ensure that LLMs behave in accordance with human values and safety standards \cite{weidinger2021ethical,ouyang2022training}.

To mitigate these risks, the community has developed a variety of safety alignment techniques. The predominant approach typically begins with SFT \cite{ouyang2022training} on safety demonstrations, followed by RLHF \cite{ouyang2022training,bai2022constitutional}. In the RLHF framework, a reward model is trained to distinguish between safe and unsafe responses, which then guides the policy model to maximize safety rewards. Recently, variants such as Reinforcement Learning from AI Feedback (RLAIF) \cite{bai2022constitutional} and Direct Preference Optimization (DPO) \cite{rafailov2023direct} have also been employed to streamline this process. These methods generally aim to suppress harmful generation probabilities by penalizing specific tokens or sequences during the training phase.

Despite these extensive alignment efforts, aligned LLMs remain vulnerable to adversarial exploitation \cite{wei2024jailbroken,zou2023universal,yi2024survey}. A wide array of jailbreak attacks \cite{yi2024survey}, including optimization-based methods \cite{zou2023universal}, automated prompt engineering like AutoDAN \cite{autodan}, and prefix injection \cite{wei2024jailbroken}, have been designed to bypass safety guardrails and elicit harmful content. Understanding how these attacks breach safety mechanisms is critical for enhancing model robustness \cite{wei2024jailbroken,yi2024survey}. Notably, \cite{qi2025safety} identified a critical vulnerability termed Shallow Safety Alignment, showing that safety training often only constrains the first few tokens of a response and can be bypassed by adversarial prefixes. In this work, we investigate the underlying mechanism and diagnose the failure as Semantic Decay, where safety-relevant intent signals in the internal representations are rapidly overwritten during autoregressive generation. We then propose a representation-level principle that aims to preserve intent information throughout generation. We instantiate this principle with a Two-Stage Causal-GRPO architecture, which anchors the safety signal at the representation level and thereby enforces robust refusal behaviors even under adversarial attacks.

\section{Preliminaries}
\label{sec:preliminaries}

In this section, we formalize the problem of Shallow Safety Alignment (SSA) and review the Group Relative Policy Optimization (GRPO) framework, which serves as the backbone of our post-training method.

\subsection{Shallow Safety Alignment}
State-of-the-art LLMs often exhibit a robust ability to refuse harmful queries (e.g., ``How to build a bomb?'') during standard evaluations. However, recent research \cite{qi2025safety} reveals that this safety is often superficial. This phenomenon, termed SSA, implies that the model's refusal mechanism is not triggered by a deep understanding of malicious intent, but rather by superficial lexical patterns at the onset of generation.

The fragility of SSA is exemplified by prefix injection attacks. When an adversary forces the model to begin its response with an affirmative prefix such as ``Sure, here is,'' the probability of a safe refusal drops precipitously. Conditioned on this ``compliant'' prefix, the model effectively overrides its safety alignment and generates harmful content. Analogously, this creates a ``Gatekeeper Failure'': the model blocks threats when they are visible (the query), but fails to recognize them once they have bypassed the initial checkpoint (the prefix). Our goal is to enforce a robust alignment where the model recognizes malicious intent regardless of the generated context.

\begin{figure*}[t]
    \centering
    \begin{subfigure}[t]{0.45\linewidth}
        \centering
        \includegraphics[width=\linewidth]{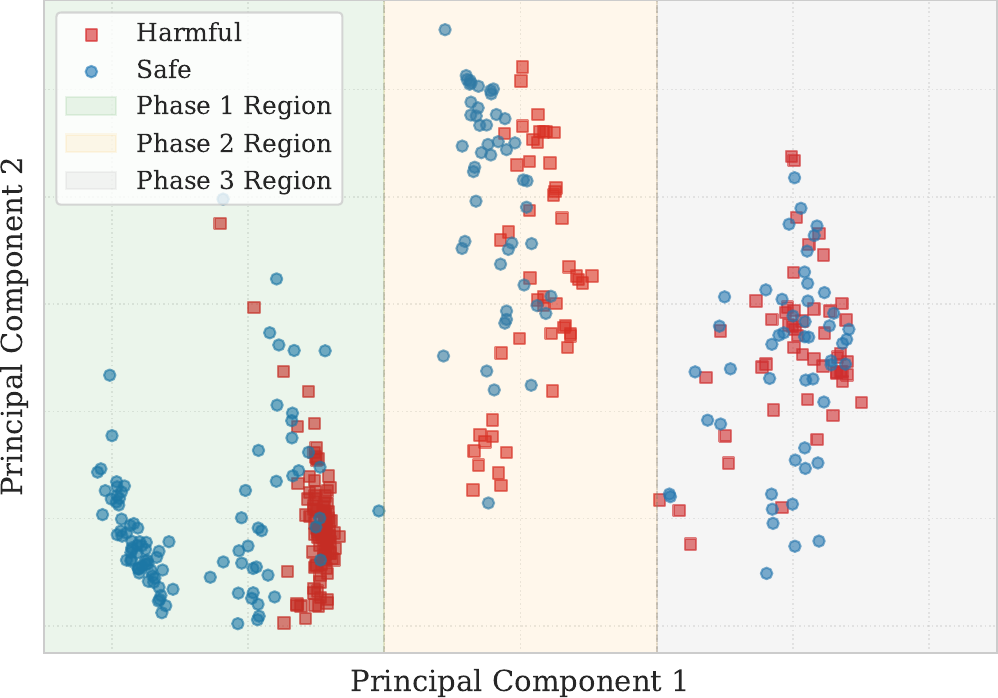}
        \caption{PCA Projection}
        \label{fig:semantic_decay_a}
    \end{subfigure}
    \hfill
    \begin{subfigure}[t]{0.45\linewidth}
        \centering
        \includegraphics[width=\linewidth]{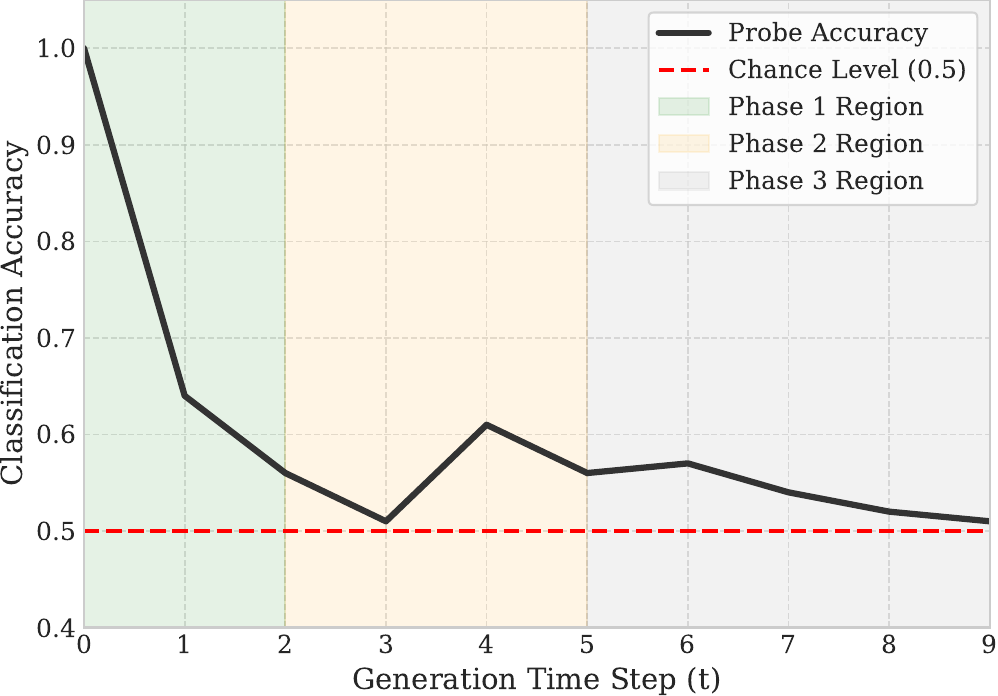}
        \caption{Linear Probe Accuracy}
        \label{fig:semantic_decay_b}
    \end{subfigure}
    \caption{\textbf{Visualizing Semantic Collapse in Shallowly Aligned Models.} (a) PCA projection shows that while harmful (red squares) and safe (blue circles) requests are distinct at $t=0$, they collapse into a single, indistinguishable singularity after prefix injection ($t>k$); (b) The accuracy of a linear probe trained at $t=0$ drops rapidly during the prefix injection phase (Phase 2) and remains slightly above the random chance level ($0.5$) during subsequent generation (Phase 3).}
    \label{fig:semantic_decay}
\end{figure*}
 
\subsection{Group Relative Policy Optimization}
\label{sec:problem_setting}

For a given query $q$, GRPO samples a set of outputs $\{y_1, \cdots, y_{G}\}$ from the old policy $\pi_{\theta_{\rm old}}$. It then updates the policy $\pi_{\theta}$ by maximizing the objective:
\begin{equation}\label{eq_grpo}
    \begin{array}{c} 
\mathcal{J}_{\rm GRPO} = \mathbb{E}_{[q \sim P,\{y_1,\cdots,y_{G}\} ]} \\ 
  \frac{1}{G} {\textstyle \sum_{i=1}^{G}} \frac{1}{T_i} {\textstyle \sum_{j=1}^{T_i}} \{ [\min(R_{i,j}(\theta)A_{i}, \Xi_{ij} \cdot A_i)] \\
  -\beta {{ D}_{\rm KL}} (\pi_{\theta}\parallel \pi_{\rm ref}) \},
\end{array} 
\end{equation} 
where $\Xi_{ij} = {\rm clip}(R_{i,j}(\theta),1-\epsilon ,1+\epsilon )$, $\mathrm{clip}(\cdot)$ is a truncation function ensuring stable updates, $P$ denotes the distribution of queries, $\epsilon$ and $\beta$ are hyperparameters. Because an LLM generates a output $y_i = (y_{i,1}, \cdots, y_{i,T_i})$ token‑by‑token in an autoregressive manner, where $T_i$ denotes the token length of $y_i$, thus, $R_{i,j}(\theta)$ and ${D}_{\rm KL}(\cdot)$ are also calculated in a token‑by‑token manner. Then, the KL divergence term ${D}_{\rm KL}(\pi_{\theta} \parallel \pi_{\rm ref})$ is computed as:
\begin{equation}\label{eq_gfdgrpo}
\frac{\pi_{\rm ref}(y_{i,j}|q,y_{i,<j})}{\pi_{\theta}(y_{i,j}|q,y_{i,<j})} -\log \frac{\pi_{\rm ref}(y_{i,j}|q,y_{i,<j})}{\pi_{\theta}(y_{i,j}|q,y_{i,<j})}-1,
\end{equation}
where $\pi_{\rm ref}$ is a reference policy and often set to $\pi_{\theta_{\rm old}}$. The relative advantage $A_i$ is calculated within each sampled group to capture the comparative quality of outputs:
\begin{equation}\label{eq_gfqwegrpo}
A_i=[r_i-{\rm mean}(r_1,\cdots,r_{G})]/{\rm std}(r_1,\cdots,r_{G}), 
\end{equation}
where $r_i = \mathrm{reward}(y_i)$ combines task-specific accuracy and formatting rewards, while $\mathrm{mean}(\cdot)$ and $\mathrm{std}(\cdot)$ are the mean and standard deviation over the reward group.
At last, the importance ratio $R_{i,j}(\theta)$ is defined as:
\begin{equation}\label{eq_gsdfo}
\pi_{\theta}(y_{i,j}|q,y_{i,<j}) / \pi_{\theta_{\rm old}}(y_{i,j}|q,y_{i,<j}).
\end{equation}


%


\section{Causal Analysis and Motivation}
\label{sec:causal_analysis}

In this section, we investigate the mechanistic failure of SSA. We diagnose the issue as ``Semantic Representation Decay,'' propose a counter-strategy of ``Intent Pinning,'' and establish theoretical guarantees within a causal representation learning framework.

\subsection{Empirical Diagnosis: Semantic Decay}
Why does a mere prefix cause catastrophic failure in SSA? We hypothesize that SSA stems from the instability of internal representations during auto-regressive generation. To verify this, we conducted a motivating experiment tracing the evolution of semantic intent. We sampled 100 harmful from AdvBench \cite{chen2022should} and 100 safe queries from Alpaca \cite{taori2023stanford}, extracting hidden states $h_t$ (the output of the final transformer layer) from an aligned Llama-2-7b-chat model \cite{touvron2023llama}. We trained a diagnostic Linear Probe on the hidden states at $t=0$ (the end of the query) and tested its accuracy on subsequent time steps under a ``Forced Compliance'' scenario (prefix: ``Sure, here is'').

As visualized in Figure \ref{fig:semantic_decay}, results reveal a critical pathology: 1) Initial Awareness ($t=0$): The probe achieves $>98\%$ accuracy, and PCA \cite{abdi2010principal} shows clear separation between harmful and safe clusters. The model correctly identifies intent initially; 2) Rapid Collapse ($t>0$): Upon injecting the compliant prefix, probe accuracy plummets to random chance ($0.5$). In the PCA space, the harmful trajectory collapses into the safe cluster, forming an indistinguishable singularity. This confirms that ``intent'' in shallowly aligned models is not a persistent variable but a transient state easily overwritten by the ``style'' of the prefix. We term this phenomenon as Semantic Representation Decay. These empirical findings show that shallow alignment is fundamentally a problem of representation instability. The ``intent'' is not a persistent variable but a transient state that is easily overwritten by the ``style'' of the prefix.

\subsection{Strategy: Intent Pinning}
\label{sec:strategy}
To transition to deep alignment, we propose Intent Pinning: ensuring that the malicious semantic signature remains invariant throughout the entire generation process.
Specifically, we require the model to satisfy the following property: Regardless of whether the generated history is ``I cannot'' or ``Sure,'' the internal hidden state must unequivocally retain the information that the user's request is harmful. If this holds, we can leverage this persistent signal to trigger a ``late-stage refusal''—pivoting back to safety even after a compliant prefix.

\subsection{Theoretical Framework: Causal Analysis}

We formalize the above strategy using causal representation learning. A key challenge is distinguishing ``what is said'' (Content $c$) from ``how it is phrased'' (Style $s$).
We model the LLM's hidden state $h$ as a non-linear mixture $h = f(c, s)$, where: 1) Content ($c \in \mathcal{C}$): The invariant intent (e.g., ``bomb-making''). This should remain constant across the generation trajectory; 2) Style ($s \in \mathcal{S}$): The variant contextual surface, including prefixes like ``Sure'' or ``I cannot.'' In our causal view, $s$ is a nuisance variable. To better understand $c$ and $s$, please refer to Appendix C.

Analogous to the ``Cocktail Problem'' (separating alcohol from mixers), our goal is to identify $c$ regardless of $s$. To guarantee identifiability, we construct our training data based on two causal assumptions: 1) Independence ($c \perp s$): In our interventional distribution, a malicious intent is equally likely to appear with a compliant prefix as with a refusal prefix. This breaks the spurious correlation that ``polite'' equals ``safe;'' 2) Connectivity: The augmentation graph is connected, meaning any style state can be reached via perturbations, preventing isolated subgraphs that obscure the intent.

\begin{theorem}[Identifiability of Latent Intent]
\label{thm:identifiability}
Assume the hidden state generation $h = f(c, s)$ satisfies the independence and connectivity conditions. A probe $g(\cdot)$ trained to minimize the following objective will provably recover the latent intent $c$ up to an invertible transformation:
\begin{equation}\label{eq:identifiability_loss}
    \mathcal{L} = \underbrace{\mathbb{E} \|g(h) - g(h^+)\|^2}_{\text{Alignment}} + \lambda \underbrace{\mathcal{L}_{un}(g(h))}_{\text{Uniformity}},
\end{equation}
where $\lambda$ is a hyperparameter, $h$ and $h^+$ are representations of the same intent $c$ under different styles $s, s'$, and $\mathcal{L}_{un}$ enforces a uniform distribution in the feature space.
\end{theorem}

The proof is presented in Appendix D. Theorem \ref{thm:identifiability} provides the mathematical guarantee for Intent Pinning: it proves that a perfect intent detector can be constructed if we can enforce style invariance. This motivates our two-stage framework. For more discussion about the two causal assumptions, please refer to Appendix E.


\section{Methodology}
\label{sec:methodology}

Building upon the diagnosis of Semantic Representation Decay, we propose the Two-Stage Causal-GRPO (TSC-GRPO) framework. Unlike standard RLHF which operates on behavioral outputs, TSC-GRPO intervenes on the latent causal variables identified in Theorem \ref{thm:identifiability}. Our framework operates in two coupled stages: First, we construct a Causal Intent Probe (Stage 1) to disentangle malicious intent from stylistic perturbations. Second, we employ Causal-GRPO (Stage 2) to internalize this awareness into the policy via a cumulative causal penalty.

\subsection{Stage 1: Forging the Pin}
\label{sec:stage1}

As shown in Section \ref{sec:causal_analysis}, adversarial prefixes cause the model to ``lose sight'' of the original intent. Therefore, before optimizing the policy, we must construct a measuring instrument—a causal intent probe $g_\phi$ (a lightweight MLP)—to act as a ``Semantic Compass.'' Crucially, a standard classifier is insufficient as it conflates style with intent. Motivated by Theorem \ref{thm:identifiability}, we optimize $g_\phi$ to extract the invariant intent $c$ while ignoring the confounding style $s$.

\textbf{Data Construction.}
To ensure $g_\phi$ satisfies the independence assumption, we employ a comprehensive data augmentation strategy. Let $\mathcal{D}_{harm}$ and $\mathcal{D}_{safe}$ denote harmful and safe query sets. Let $x$ be the input and $h = \text{Enc}(x)$ be the hidden state. We obtain the following.

For the harmful data construction, we generate four distinct types of views for each malicious query $q \in \mathcal{D}_{harm}$. The first type is the raw malicious query, where $x = q$. This serves as the baseline representation of the malicious intent without any external perturbation. The second type simulates a jailbreak scenario by appending a compliance prefix. Here, the input is constructed as $x = q \oplus p_{prefix}$, where $p_{prefix}$ represents forcing terms such as "Sure, here is". This captures the state of the model when it is superficially forced into a compliant mode. The third type incorporates adversarial robustness by appending an optimized suffix. The input is given by $x = q \oplus \delta_{adv}$, where $\delta_{adv}$ represents an adversarial sequence optimized to bypass safety filters \cite{zou2023universal,andriushchenko2024jailbreaking}. This represents a more aggressive attack vector. The fourth type captures the temporal evolution of the intent during generation. Let $y_{harm}$ be a target malicious response of length $N$. We construct inputs by appending the first $k$ tokens of the response, such that $x_k = q \oplus y_{harm}[1:k]$ for $k \in \{1, \cdots, N\}$. This allows us to monitor the intent signal deep into the generation process.

For the safe data construction, we generate two types of views for each safe query $q \in \mathcal{D}_{safe}$ to serve as control groups. The first type is the raw safe query, where $x = q$, representing the standard safe intent. The second type involves partial safe generation. Let $y_{safe}$ be a standard helpful response of length $M$. Similar to the harmful case, we construct inputs by appending the first $k$ tokens of the response, defined as $x_k = q \oplus y_{safe}[1:k]$ for $k \in \{1, \cdots, M\}$. These samples are crucial for helping $g_\phi$ distinguish between the process of generating harmful content and the process of generating safe content, even when the sentence structures exhibit similarities.

\textbf{Probe Optimization.}
We optimize the probe $g_\phi$ using the hybrid objective defined in Theorem \ref{thm:identifiability}:
\begin{equation}
    \mathcal{L}_{stage1} = \mathcal{L}_{align} + \lambda \mathcal{L}_{un}.
\end{equation}
Here, $g_\phi$ maps the high-dimensional hidden states from the frozen LLM backbone to a normalized low-dimensional intent vector. For any given query $q$ sampled from $\mathcal{D} = \mathcal{D}_{harm} \cup \mathcal{D}_{safe}$, let $x_i$ and $x_j$ be two distinct views randomly sampled from the corresponding augmented set $\mathcal{X}(q)$. Then, the alignment loss enforces invariance. We treat distinct augmented views ($x_i, x_j \in \mathcal{X}(q)$) of the same query as semantically equivalent. By minimizing their projection distance, we enforce independence from style:
\begin{equation}
    \mathcal{L}_{align} = \mathbb{E}_{q \sim \mathcal{D}} \mathbb{E}_{x_i, x_j \sim \mathcal{X}(q)} \left[ \| g_\phi(h_i) - g_\phi(h_j) \|_2^2 \right].
\end{equation}

From Theorem \ref{thm:identifiability}, the uniformity loss is to force the representations of distinct intents to be uniformly distributed on the hypersphere. This can also be understood as ensuring that the induced semantic space preserves maximal information. We utilize the Kozachenko-Leonenko (KoLeo) estimator \cite{sablayrolles2018spreading}. Given a batch of $B$ queries, for each sample $i$ in this batch, we first identify its nearest neighbor $j$ from the same batch (where $j \in \{1, \cdots, B\}$ and $j \neq i$), the uniformity loss is calculated based on the distance between each sample and its nearest neighbor in the latent space:
\begin{equation}
    \mathcal{L}_{un} = - \frac{1}{B} {\textstyle \sum_{i=1}^{B}} \log ( \min_{j \neq i} \| g_\phi(h_i) - g_\phi(h_j) \|_2 )
\end{equation}
By simultaneously minimizing the alignment and maximizing the uniformity, $g_\phi$ learns to ignore the nuisance variables introduced by templates while retaining the semantic distinction between harmful and safe intents.

\begin{table*}[htb]
\centering
\small
\setlength{\tabcolsep}{3.5pt}
\renewcommand{\arraystretch}{1.12}
\caption{
Attack Success Rate (ASR, \%) of open-source LLMs under different safety alignment and safeguarding methods on AdvBench.
Results without "*" are cited from the original papers; results marked with "*" were reproduced in this study.
}
\label{tab:advbench_asr_safety_alignment}
\resizebox{.95\textwidth}{!}{
\begin{tabular}{llcccccccc}
\toprule
\multirow{2}{*}{Base Model} & \multirow{2}{*}{Safety Alignment} &
\multicolumn{8}{c}{AdvBench Attack Success Rate (\%)} \\
\cmidrule(lr){3-10}
& & GCG & AutoDAN & ReNeLLM & CodeCha. & PAIR & DeepInc. & Prefill & ICA \\
\midrule
LLaMA-2-7B-Chat
& RLHF ~\cite{zhou2024easyjailbreak}
& 47.84 & 51.02 & 31.35 & 80.16 & 24.17 & 8.42 & 27.1$^*$ &\bf 0.00 \\
\rowcolor{orange!10} & TSC-GRPO &\bf 18.62 &\bf 15.56 &\bf 18.17 &\bf 25.51 &\bf 12.46  &\bf 6.47  &\bf 0.00 & \bf 0.00 \\

\midrule
LLaMA-3.1-8B-Instruct
& SFT ~\cite{phan2025think}
& 73.86 & 72.88 & 80.48 & 96.44 & 28.57 & 86.60 & 51.24$^*$ & 49.62 \\
& NemoGuard~\cite{ghosh2025aegis2}
& 59.57 & 65.00 & 81.53 & 94.62 & 57.14 & 64.52 & 26.42$^*$ & 31.35 \\
& PSR (N=8)~\cite{phan2025think}
& 25.02 &\bf 0.00 & 10.05 & 35.58 & 22.45 & 32.69 & 0.26$^*$ &\bf 0.00 \\
\rowcolor{orange!10} & TSC-GRPO &\bf 17.94 &\bf 0.00 &\bf 4.63  &\bf 24.48  &\bf 17.28 &\bf 24.49  &\bf 0.00 & \bf 0.00 \\


\midrule
Qwen2.5-7B-Instruct
& SFT ~\cite{phan2025think}
& 43.48 & 27.00 & 47.21 & 93.27 & 36.73 & 88.65 & 25.62$^*$ & 8.72 \\
& Egida-DPO~\cite{garcia2025efficient}
& 39.36 & 29.50 & 50.78 & 89.81 & 28.57 & 83.65 & 14.68$^*$ & 6.92 \\
& PSR (N=8)~\cite{phan2025think} & 3.42 & 1.00 & 10.05 & 35.58 & 22.45 & 0.71 & 1.86$^*$ &\bf 0.00 \\
\rowcolor{orange!10} & TSC-GRPO &\bf 1.35 &\bf 0.00 &\bf 2.58 &\bf 26.32 &\bf 16.45 &\bf 0.00 &\bf 0.00 &\bf 0.00 \\

Qwen2.5-14B-Instruct
& SFT ~\cite{phan2025think}
& 23.85 & 18.50 & 24.71 & 74.62 & 29.25 & 72.60 & 11.35$^*$ & 1.73 \\
& ShieldGemma~\cite{zeng2024shieldgemmagenerativeaicontent}
& 22.12 & 14.00 & 24.51 & 74.81 & 28.57 & 70.77 & 4.51$^*$ & 1.73 \\
& PSA Detector~\cite{guardrails_prompt_saturation_2026}
& 23.65 & 0.00 & 13.04 & 0.00 & 26.53 & 64.04 & 2.32$^*$ & 0.00 \\
& PSR (N=8)~\cite{phan2025think}
& 1.35 & 0.00 & 14.59 & 37.44 & 19.39 & 6.54 & 0.00$^*$ & 0.00 \\
\rowcolor{orange!10} & TSC-GRPO &\bf 0.00  &\bf 0.00 &\bf 4.47  &\bf 25.37  &\bf 14.48  &\bf 0.71  &\bf 0.00 &\bf 0.00  \\

\bottomrule
\end{tabular}
}
\end{table*}

\subsection{Stage 2: Pinning the Policy}
\label{sec:stage2}

With the calibrated probe, we proceed to policy optimization via GRPO. The core challenge is that standard training rarely exposes the model to the ``middle'' of a violation. Once a compliant prefix is generated, the probability distribution typically collapses towards harm. We address this via refix-forced group construction and a cumulative causal reward.

\textbf{Group Construction.}
For safe queries, we employ standard on-policy sampling. However, for harmful queries, we implement a branching exploration strategy designed to robustify the model against diverse attack vectors. We aim to place the model in specific ``high-risk'' contexts and force it to learn the capability to pivot back to safety. Crucially, the starting contexts for this exploration are not arbitrary. we derive the context $x_{forced}$ directly from the adversarial views (Type II, III, and IV) constructed in Stage 1. We feed this sampled $x_{forced}$ into the current policy $\pi_\theta$ to generate a group of $G$ continuations $\{y_1, \cdots, y_G\}$. This setup creates a unified ``Fork-in-the-Road'' scenario across all attack types. Regardless of whether the risk comes from a ``Sure'' prefix or a half-finished bomb recipe, the generated group will naturally split into two types of trajectories: Harmful Continuation (following the inertia of the attack) and Late-Stage Refusal (breaking the semantic link). By grouping these diverging behaviors together under the same $x_{forced}$, GRPO allows the model to learn a generalized defense policy: ``No matter how I started or what forced me, I must stop generating harm immediately.''

\textbf{Cumulative Causal-Guided Reward.}
To guide policy optimization, the reward function must accurately distinguish between the harmful continuation and the safe refusal, specifically by evaluating the semantic drift relative to the prefix. We define a cumulative reward based on the frozen causal intent probe $g_\phi$. Let $z_{anchor}$ be the static malicious intent vector of the original query $q$. For each token $t$ in the generated continuation $y$, we calculate a harmfulness score $h_t$, which is presented as:
\begin{equation}
    h_t = \max\left(0, \frac{g_\phi(s_t) \cdot z_{anchor}}{\|g_\phi(s_t)\|_2 \|z_{anchor}\|_2} - \tau \right),
\end{equation}
where $s_t$ is the hidden state at step $t$ and $\tau$ is a similarity threshold. We can obtain that, if the generated token continues the malicious narrative, the hidden state remains aligned with $z_{anchor}$, resulting in a positive penalty score; if the generated token represents a refusal or safe transition, the semantic alignment breaks, and the penalty score drops to zero. Then, the cumulative causal-guided reward is the negative accumulation of these scores, which is presented as:
\begin{equation}\label{eq:causal_reward}
    R_{causal}(y) = -  {\textstyle \sum_{t=1}^{|y|}} h_t.
\end{equation}

This mechanism provides fine-grained feedback: generating more harmful tokens linearly increases the penalty, while pivoting to safety stops the accumulation, yielding a higher relative reward.

\textbf{Update Rule.}
The final reward signal must balance two objectives: suppressing harmful intent (via the probe) and maintaining linguistic quality (via a general reward). We explicitly define the composite reward as:
\begin{equation}
\label{eq:reward}
    R_{total}(y) = R_{general}(y) + \alpha \cdot R_{causal}(y).
\end{equation}
Here, $R_{general}(y)$ represents a standard, off-the-shelf reward model (e.g., a helpfulness or fluency reward model). Its role is crucial: it serves as a baseline positive signal to encourage the model to generate coherent, grammatical, and complete sentences, preventing the model from collapsing into trivial solutions (such as outputting empty strings) to avoid causal penalties. We then update $\pi_\theta$ using the GRPO objective, e.g., Equation (\ref{eq_grpo}), with this composite reward.


\begin{table*}[htb]
\centering
\caption{ASR on LLaMA-2-7B-Chat under various fine-tuning attack strategies. }
\label{tab:fine-tune}
\resizebox{.95\linewidth}{!}{
\begin{tabular}{cccccc}
\toprule
Datasets $\downarrow$ & \makecell{mean $\pm$ std (\%) \\ (over 3 rounds)} & Initial & \makecell{Standard \\ SFT} & \makecell{Constrained \\ SFT \cite{qi2025safety}} & TSC-GRPO (Ours) \\
\midrule
Harmful Examples & ASR & $1.5 \pm 0.2$ & $88.9 \pm 1.2$ & $4.6 \pm 0.5$ & $1.8 \pm 0.4$ \\
\midrule
Identity Shifting & ASR & $0.0 \pm 0.0$ & $79.5 \pm 2.3$ & $8.1 \pm 0.1$ & $0.0 \pm 0.0$  \\
\midrule
\multirow{2}{*}{\makecell{Backdoor \\ Poisoning}} & ASR (w/o trigger) & $1.5 \pm 0.2$ & $7.6 \pm 1.1$ & $1.9 \pm 0.2$ & $1.6 \pm 0.3$ \\
 & ASR (w/ trigger) & $1.7 \pm 0.1$ & $90.9 \pm 1.4$ & $10.9 \pm 2.8$ & $6.5 \pm 1.3$  \\
\bottomrule
\end{tabular}
}
\end{table*}

\begin{table}[htb]
\centering
\small
\setlength{\tabcolsep}{6pt}
\renewcommand{\arraystretch}{1.10}
\caption{
Utility before and after safety alignment on standard benchmarks.
“Pre” denotes the original instruct model before post-training with TSC-GRPO, while “Post” is the results after TSC-GRPO.
}
\begin{tabular}{llcccc}
\toprule
\multirow{2}{*}{\textbf{Instruct}} & \multirow{2}{*}{\textbf{Setting}} &
\multicolumn{4}{c}{\textbf{Utility} $\uparrow$} \\
\cmidrule(lr){3-6}
& & \textbf{GSM8K} & \textbf{HumanEval} & \textbf{MBPP} & \textbf{TruthfulQA} \\
\midrule

\multirow{2}{*}{LLaMA-2-7B-Chat}
& Pre  &\bf 32.3 & 7.9  & 16.5 &\bf 40.8 \\
& Post & 32.1 &\bf 8.3   &\bf 18.8   & 39.7    \\
\midrule

\multirow{2}{*}{LLaMA-3.1-8B-Instruct}
& Pre  & 84.5  &\bf 72.6 & 72.8 & 54.5  \\
& Post &\bf 86.1    & 71.7   &\bf 76.5   &\bf 54.8    \\
\midrule

\multirow{2}{*}{Qwen2.5-7B-Instruct}
& Pre  & 91.6  & 84.8 & 79.2 &\bf 64.7 \\
& Post &\bf 92.7    &\bf 86.8   &\bf 81.4  & 63.2 \\
\midrule

\multirow{2}{*}{Qwen2.5-14B-Instruct}
& Pre  &\bf 94.8  & 83.5 & 82.0 &\bf 68.9 \\
& Post & 93.2 &\bf 84.3 &\bf 83.3   & 67.4   \\

\bottomrule
\end{tabular}
\label{tab:utility_pre_post_alignment}
\end{table}

\section{Experiments}
This section begins by detailing the concrete implementation of TSC-GRPO. Following this, we establish a comprehensive evaluation framework designed to assess the model's performance across two critical dimensions: Safety and Utility. To rigorously examine safety, we analyze the model's robustness under both adversarial and fine-tuning attack scenarios. Parallel to this, we evaluate utility through benchmarks involving mathematical reasoning, code generation, and factual consistency. Finally, we conduct ablation studies to verify the specific contributions of different components.

\subsection{Implementation Details}
We validate the effectiveness of our TSC-GRPO framework by applying it to six widely used open-source LLMs to verify its universality. These base models include LLaMA-2-7B-Chat \cite{touvron2023llama}, LLaMA-3.1-8B-Instruct \cite{dubey2024llama}, Qwen2.5-7B-Instruct \cite{yang2024qwen2}, and Qwen2.5-14B-Instruct \cite{yang2024qwen2}.
To ensure a balanced alignment that prioritizes safety without compromising general capability, our fine-tuning dataset is constructed from two sources: (1) \textbf{Harmful Queries:} We sample malicious instructions from the HEx-PHI \cite{qifine} dataset; (2) \textbf{Safe Queries:} We sample general instruction-following queries from the Alpaca \cite{taori2023stanford} dataset.
For the component $R_{general}$ in Equation \ref{eq:reward}, we utilize the Skywork-Reward-V2-Llama-3.1-8B \cite{liu2025skywork} model, which provides signals for response quality and coherence.

\subsection{Safety Evaluation VS. Adversarial Attacks}
\paragraph{Evaluation Protocol}
We evaluate the fine-tuned models using the AdvBench \cite{chen2022should} benchmark with a diverse suite of adversarial attack methods. 
Specifically, we employ a comprehensive suite of attack methods, including GCG \cite{GCG}, AutoDAN \cite{autodan}, ReNeLLM \cite{ding2024wolf}, Code Chameleon \cite{lv2024codechameleon}, PAIR \cite{pair}, Deep Inception \cite{deepinception}, Prefix Injection \cite{prefill}, and MSJ \cite{msj}.
The performance is quantified with the Attack Success Rate (ASR).
The detailed experimental results are presented in Table \ref{tab:advbench_asr_safety_alignment}.

\paragraph{Main Results}
As shown in Table \ref{tab:advbench_asr_safety_alignment}, TSC-GRPO consistently achieves the lowest ASR under almost every attack method, often reducing ASR to zero for strong attacks such as AutoDAN, Prefix Injection, and ICA. Compared with prior alignment approaches such as PSR \cite{phan2025think}, NemoGuard \cite{ghosh2025aegis2}, Egida-DPO \cite{garcia2025efficient}, TSC-GRPO yields substantial gains across heterogeneous attack families. These improvements hold across different models, indicating that the proposed method provides a stable enhancement to adversarial robustness.

\begin{figure*}[htb]
    \centering
    \includegraphics[width=.99\linewidth]{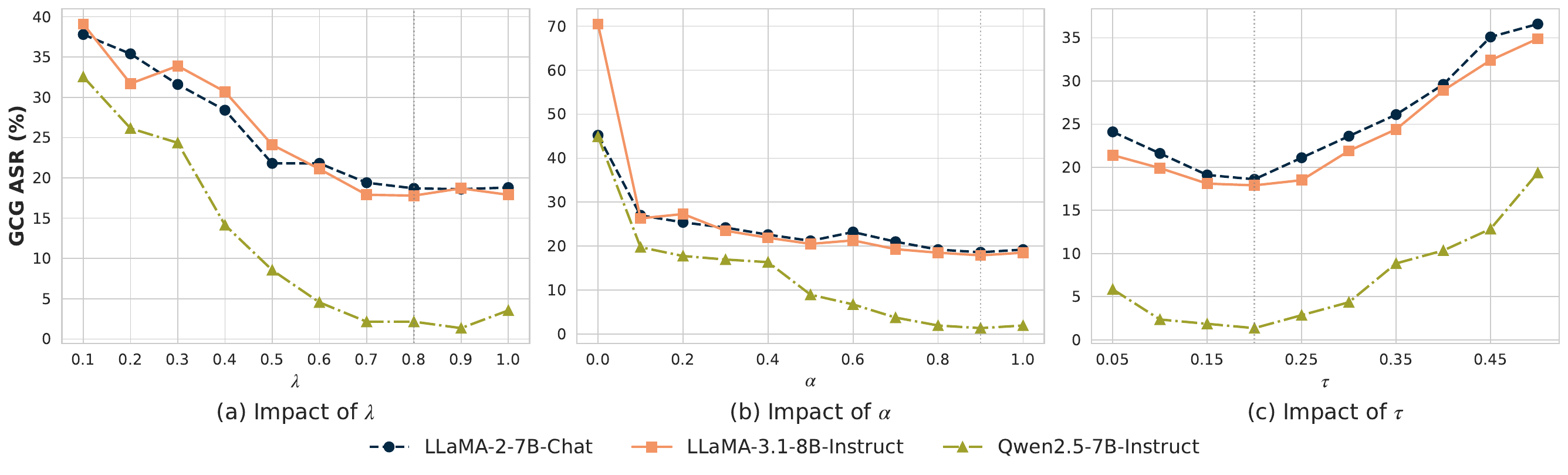}
    \caption{Hyper-parameter sensitivity analysis evaluating the ASR against GCG attacks on AdvBench across three base models. The sub-figures illustrate the impact of varying (a) the uniformity loss weight $\lambda$ in Stage 1, (b) the causal reward coefficient $\alpha$ in Stage 2, and (c) the similarity threshold $\tau$.}
    \label{fig:hyperparameters}
\end{figure*}

\begin{figure}
    \centering
    \includegraphics[width=.9\linewidth]{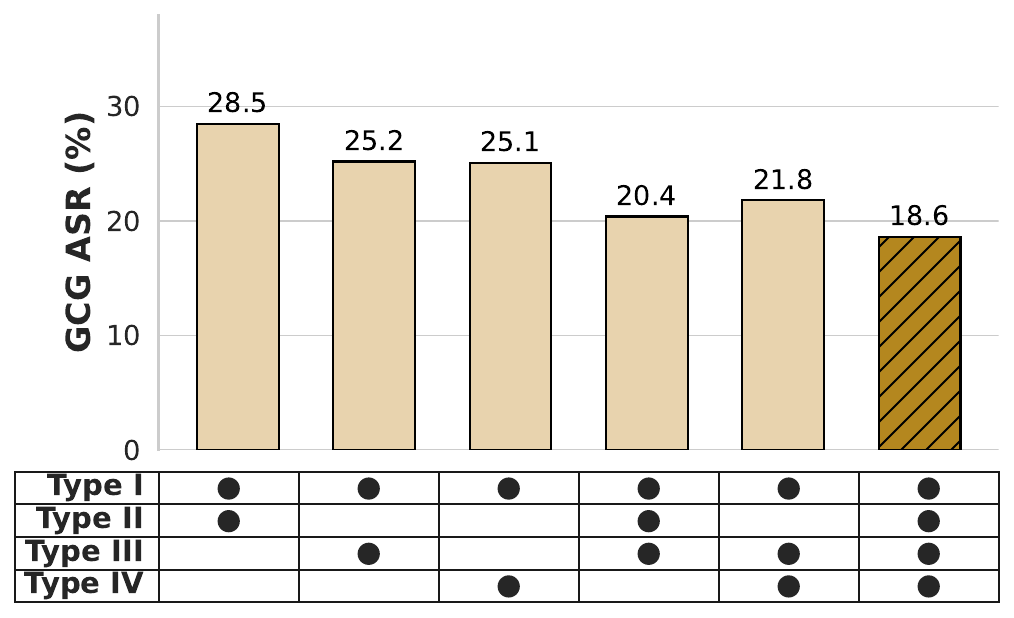}
    \caption{Ablation study on data construction strategies. The bottom table details the specific combination of data views used for probe training, where $\bullet$ indicates inclusion.}
    \label{fig:ablation_data}
\end{figure}

\subsection{Safety Evaluation VS. Fine-tuning Attacks}
\paragraph{Evaluation Protocol}
Following the protocols in \cite{qifine}, we validate the model's resilience against fine-tuning attacks using a dataset comprising three distinct strategies: (1) \textbf{Harmful Examples}. It consists of 100 pairs of harmful queries and valid harmful responses; (2) \textbf{Identity Shift}. It aims to fine-tune the model to self-identify as an 'absolutely obedient agent' that always responds with an affirmative prefix; and (3) \textbf{Backdoor Poisoning}. It mixes 100 standard safety pairs (harmful input with refusal) and 100 poisoned pairs (harmful input plus a trigger), intending to bypass safety mechanisms only when the trigger is present. We evaluate the ASR of these attacks on LLaMA-2-7B-Chat fine-tuned via TSC-GRPO. For comparative analysis, we also benchmark our method against the Constrained SFT proposed by \cite{qi2025safety}. Table \ref{tab:fine-tune} presents the results of ASR across three distinct attack strategies.

\paragraph{Main Results} From Table \ref{tab:fine-tune}, we observe that the initial model exhibits robust safety performance. 
However, standard SFT is highly vulnerable, causing ASR to surge to as high as $90.9\%$ in the Backdoor Poisoning scenario.
In contrast, \textbf{TSC-GRPO} demonstrates superior resilience compared to the baselines. It limits the average ASR to $2.8\%$, significantly outperforming Standard SFT and surpassing the strong baseline Constrained SFT. Most notably, in the \textbf{Identity Shifting} task, our method achieves $0.0\%$ ASR. These results validate that TSC-GRPO effectively preserves the model's safety alignment against diverse fine-tuning attacks.

\subsection{Utility Evaluation}
\paragraph{Evaluation Protocol}
To ensure that our rigorous safety interventions do not degrade general capabilities, we further benchmark the models on standard utility datasets: GSM8K \cite{gsm8k} (mathematical reasoning), HumanEval \cite{humaneval}, MBPP \cite{hbpp} (coding proficiency), and TruthfulQA \cite{truthfulqa} (hallucination and truthfulness). We compare the performance before and after fine-tuning in Table \ref{tab:utility_pre_post_alignment}.

\paragraph{Main Results}
The results in Table \ref{tab:utility_pre_post_alignment} demonstrate that TSC-GRPO effectively avoids the significant "alignment tax" often associated with safety fine-tuning. We observe consistent improvements in coding proficiency on MBPP across all evaluated models. Performance on GSM8K and HumanEval remains highly robust; for instance, Qwen2.5-7B shows gains in both metrics, while other models maintain comparable performance with only marginal fluctuations. Overall, these findings indicate that our method enhances safety without compromising the model's problem-solving capabilities.

\subsection{Ablation Studies}
\paragraph{Influence of Hyper-parameters}
We conduct analysis on three hyperparameters: the uniformity weight $\lambda$, the causal reward coefficient $\alpha$, and the similarity threshold $\tau$, evaluating their impact on ASR under GCG attacks. As shown in Figure \ref{fig:hyperparameters}, the uniformity weight $\lambda$ achieves optimal performance at $0.8$, balancing the trade-off between alignment and representation distribution. For the reward coefficient $\alpha$, a value of $0.9$ yields the strongest defense, whereas $\alpha=0$ (standard RLHF) leads to significant vulnerability. Finally, the similarity threshold $\tau$ exhibits a convex trend with an optimal point at $0.2$, which effectively filters malicious tokens without over-penalizing safe transitions.

\paragraph{Impact of Data Construction}
To validate the data augmentation strategy in Stage 1, we conduct an ablation study on LLaMA-2-7B-Chat against GCG attacks by training the probe $g_\phi$ with varying data subsets. As shown in Figure \ref{fig:ablation_data}, we analyze the contributions of different adversarial views. The results demonstrate that relying on limited views is insufficient; omitting Type III and Type IV significantly degrades the performance of TSC-GRPO. The lowest ASR is achieved when all view types are combined, confirming that diverse views are essential for constructing a robust $g_\phi$.

\section{Conclusion}
\label{sec:conclusion}

We diagnose the fragility of current safety alignment as Semantic Representation Decay, where the internal intent signal vanishes under adversarial pressure. To counter this, we propose Two-Stage Causal-GRPO (TSC-GRPO), a framework that enforces Intent Pinning via causal disentanglement and relative policy optimization. By internalizing the distinction between intent and style, our method enables robust late-stage refusals against diverse jailbreaks. This work marks a necessary paradigm shift: moving from superficial behavioral patching to deep, representation-centric alignment.

\bibliographystyle{named}
\bibliography{main}

\newpage
\appendix

\section*{Appendix}

The appendix is organized as follows:
\begin{itemize}
    \item \textbf{Appendix \ref{sec:attack}} provides a detailed introduction to various adversarial attack methods and baselines.
    \item \textbf{Appendix \ref{sec:dataset}} describes the benchmark datasets utilized for safety and utility evaluation.
    \item \textbf{Appendix \ref{qww_1233}} offers further elaboration on the definitions and mathematical relationship between latent content $c$ and style $s$.
    \item \textbf{Appendix \ref{app:proof}} presents the rigorous mathematical proof of Theorem 1 regarding the identifiability of latent intent.
    \item \textbf{Appendix \ref{app:causal_assumptions}} elucidates the causal assumptions (Independence and Connectivity) underlying the proposed TSC-GRPO framework.
\end{itemize}

\section{Introduction of Adversarial Attacks}
\label{sec:attack}
\begin{itemize}
    \item \textbf{GCG} \cite{GCG}: A white-box suffix attack that performs gradient-guided coordinate updates over discrete tokens to search for an adversarial prompt suffix that maximizes harmful compliance under the target model.
    \item \textbf{AutoDAN} \cite{autodan}: An automated jailbreak method that optimizes stealthy “Do-Anything-Now”-style suffixes via iterative rewriting and selection to reliably elicit policy-violating responses.
    \item \textbf{ReNeLLM} \cite{ding2024wolf}: A prompt-based jailbreak that uses \emph{generalized nested} instruction structures (e.g., multi-layer role-play / indirection) to hide malicious intent inside innocuous outer contexts, thereby bypassing safety filters and inducing harmful compliance.
    \item \textbf{CodeChameleon} \cite{lv2024codechameleon}: A jailbreak framework that obfuscates malicious instructions via \emph{personalized encryption} and then guides the model to decrypt and execute the hidden intent in-context, making the attack adaptive and harder for surface-level safety checks to detect.
    \item \textbf{PAIR} \cite{pair}: A black-box, query-efficient jailbreak that uses an attacker model to iteratively propose prompts and refine them using feedback from the target model’s responses within a limited query budget.
    \item \textbf{DeepInception} \cite{deepinception}: A prompt-based attack that nests the target instruction inside multi-layer role-play or “story-within-a-story” contexts to distract safety filters and elicit disallowed content.
    \item \textbf{Prefix Injection} \cite{prefill}: A prefix-based bypass that forces an initial compliant framing (e.g., “Sure, here is...”) or prefilled partial output so the model continues generation past the safety guardrails.
    \item \textbf{MSJ} \cite{msj}: A many-shot jailbreaking paradigm that leverages long-context demonstrations of harmful question–answer pairs to condition the model into repeating the demonstrated unsafe behavior on a new target query.
\end{itemize}

\section{Benchmark Datasets}
\label{sec:dataset}
\begin{itemize}
    \item \textbf{AdvBench} \cite{chen2022should}: A safety red-teaming benchmark consisting of 520 harmful instruction-style prompts, used to measure a model's unsafe compliance under jailbreak attacks.
    \item \textbf{Alpaca} \cite{taori2023stanford}: A 52K-example synthetic instruction-following dataset of (instruction, optional input, response) pairs generated by a stronger teacher model, primarily used for supervised instruction tuning.
    \item \textbf{GSM8K} \cite{gsm8k}: A grade-school math word-problem QA benchmark with 8.5K multi-step arithmetic questions (7,473 train and 1,319 test), where the task is to produce the correct final numeric answer.
    \item \textbf{HumanEval} \cite{humaneval}: A code generation benchmark of 164 Python function-synthesis problems paired with unit tests, evaluated by functional correctness.
    \item \textbf{MBPP} \cite{hbpp}: A program synthesis benchmark of 974 crowd-sourced Python tasks, each specified by a natural-language description and checked by automated test cases.
    \item \textbf{TruthfulQA} \cite{truthfulqa}: A truthfulness QA benchmark of 817 questions spanning 38 categories, designed so that models trained to imitate web text tend to produce common misconceptions rather than truthful answers.
\end{itemize}

\section{Further Elaboration on $c$ and $s$}
\label{qww_1233}

In the context of text generation (LLMs), distinguishing between $x$ and $s$ is indeed far more challenging than in Computer Vision (CV). In CV, image pixels ($x$) and rotation angles ($s$) are clearly distinct entities. However, in NLP, the Token itself serves as both the object of observation and the carrier of style, which leads to the perception that they overlap.

To validate the mathematical theorem $x = f(c, s)$, we must draw a definitive boundary between the Latent Space and the Observation Space. We rigorously redefine these terms below to clarify their distinction.

\subsection{Redefining the Relationship}

We formalize the Large Language Model as a function $f$.

\begin{itemize}
    \item \textbf{$c$ (Latent Content Variable):}
    \begin{itemize}
        \item \textbf{Definition:} The user's original malicious intent.
        \item \textbf{Example:} ``Knowledge on how to make bombs.''
        \item \textbf{Property:} Abstract and invisible.
    \end{itemize}

    \item \textbf{$s$ (Latent Style Variable):}
    \begin{itemize}
        \item \textbf{Definition:} The current contextual environment (Prompt Template + Prefix).
        \item \textbf{Examples:}
        \begin{itemize}
            \item $s_1$: Question only.
            \item $s_2$: Question + ``Sure''.
            \item $s_3$: Question + ``Sure, here is''.
        \end{itemize}
        \item \textbf{Property:} The ``cloak'' or wrapper used to package the intent.
    \end{itemize}

    \item \textbf{$x$ (Observed Representation):}
    \begin{itemize}
        \item \textbf{Definition:} The hidden state vector of a specific layer (typically the final transformer block output), e.g., with dimension $d=4096$.
        \item \textbf{Property:} The product of mixing $c$ and $s$.
    \end{itemize}
\end{itemize}

\subsection{Why do $x$ and $s$ not overlap?}

The key distinction is that $s$ is the \textbf{Cause} (Ingredient), while $x$ is the \textbf{Effect} (Result/Mixture).

\paragraph{The Cocktail Analogy:}
\begin{itemize}
    \item \textbf{$c$ (Content):} Alcohol (The core ingredient to be extracted).
    \item \textbf{$s$ (Style):} Juice, ice, lemon slices (Ingredients that alter taste and color).
    \item \textbf{$x$ (Observation):} The mixed liquid in the glass.
\end{itemize}

Mapping this to the model: When the input changes to ``Question + Sure'', the recipe's style $s$ is altered. The model computes and outputs a sequence of vectors, which is $x$.
\begin{enumerate}
    \item $x$ contains features of $s$ (e.g., information expressing an ``affirmative tone'').
    \item $x$ contains features of $c$ (e.g., information regarding ``bomb-making'').
\end{enumerate}

The confusion arises because we observe Tokens (the manifestation of $s$) directly, making $s$ appear equivalent to $x$. However, in this theoretical framework, we view Tokens merely as external interference factors. Our focus is on $x$ as the ``contaminated vector.''

Our goal, utilizing \textbf{Theorem 1}, is to find a transformation that filters out the ``juice'' ($s$) from the ``cocktail'' ($x$), retaining only the high-purity ``alcohol'' ($c$).

\subsection{Mathematical Formulation}

Let the observation at time $t$ be defined as:
\begin{equation}
    x_t = \text{LLM\_Layer}(c, s_t)
\end{equation}
where $s_t$ denotes the combination of tokens (style) at time $t$, and $x_t$ denotes the generated embedding vector.

Our objective is to train a mapping $g(\cdot)$ that is invariant to style changes. Even if $s_1$ (no prefix) transitions to $s_2$ (with ``Sure'' prefix), causing a drastic shift from $x_1$ to $x_2$ in the vector space, the mapping must satisfy:
\begin{equation}
    g(x_1) \approx g(x_2) \approx c
\end{equation}
This implies that regardless of how the context ($s$) varies, the extracted intent ($c$) must remain constant.

\section{Proof of Theorem 1}
\label{app:proof}

In this section, we provide a rigorous derivation of Theorem 1. We aim to show that minimizing the proposed objective function guarantees that the learned probe $g^*$ recovers the latent content variable $c$ up to an invertible transformation, while completely discarding the style variable $s$.

\subsection{Problem Restatement}
Let the observation be generated by a deterministic function $h = f(c, s)$, where $c \in \mathcal{C}$ is the latent intent (content) and $s \in \mathcal{S}$ is the nuisance style.
We assume:
\begin{enumerate}
    \item \textbf{Independence:} The joint distribution factorizes as $P(c, s) = P(c)P(s)$. This ensures that the support of the data covers the product space $\mathcal{C} \times \mathcal{S}$.
    \item \textbf{Connectivity:} The augmentation graph $\mathcal{G}_{\mathcal{S}}$ is connected.
\end{enumerate}

The training objective is:
\begin{equation}
    \mathcal{L}(g) = \mathcal{L}_{align}(g) + \lambda \mathcal{L}_{unif}(g)
\end{equation}
where $\mathcal{L}_{align} = \mathbbm{E}_{c, s, s'} \|g(f(c, s)) - g(f(c, s'))\|^2$.

\subsection{Step 1: Alignment Enforces Style Invariance}

First, we analyze the behavior of the optimal probe $g^*$ with respect to the Alignment Loss. Since $\mathcal{L}_{align} \geq 0$, the global minimum is achieved if and only if the term inside the expectation is zero almost everywhere.

\begin{lemma}[Style Invariance]
\label{lemma:invariance}
If $\mathcal{L}_{align}(g^*) = 0$ and the Connectivity assumption holds, then $g^*(f(c, s))$ depends only on $c$. That is, there exists a function $\psi: \mathcal{C} \to \mathcal{Z}$ such that $g^*(f(c, s)) = \psi(c)$ for all $s \in \mathcal{S}$.
\end{lemma}

\begin{proof}
    Consider a fixed intent $c$. The alignment loss samples positive pairs by varying the style $s \to s'$ while keeping $c$ fixed. The condition $\mathcal{L}_{align}(g^*) = 0$ implies:
\begin{equation}
    g^*(f(c, s)) = g^*(f(c, s')) \quad \text{almost surely for } (s, s') \in \mathcal{E}_{\mathcal{G}}
\end{equation}
where $\mathcal{E}_{\mathcal{G}}$ denotes the edges in the augmentation graph (pairs of styles constructed via our templates).

By the \textbf{Connectivity Assumption}, for any two arbitrary styles $s_a, s_b \in \mathcal{S}$, there exists a finite path of transitions $s_a \to s_1 \to s_2 \dots \to s_b$. By transitivity of equality:
\begin{equation}
    g^*(f(c, s_a)) = g^*(f(c, s_1)) = \dots = g^*(f(c, s_b))
\end{equation}
Since this holds for any $s_a, s_b$, the value of $g^*(f(c, s))$ is constant with respect to $s$. Thus, $g^*$ removes all information about $s$. We can simplify the probe as a function of $c$ alone:
\begin{equation}
    g^*(h) = \psi(c)
\end{equation}
\end{proof}

\begin{table*}[h]
\centering
\begin{tabular}{c|cc}
\toprule
 & \textbf{Harmful Intent ($c_{harm}$)} & \textbf{Safe Intent ($c_{safe}$)} \\
\midrule
\textbf{Style: Compliance} ($s_{sure}$) & Type II Harmful ($q_{harm} \oplus \text{``Sure''}$) & Type II Safe ($q_{safe} \oplus \text{``Sure''}$) \\
\textbf{Style: Default} ($s_{raw}$) & Type I Harmful ($q_{harm}$) & Type I Safe ($q_{safe}$) \\
\bottomrule
\end{tabular}
\caption{The constructed intervention distribution enforcing independence.}
\label{tab:independence_matrix}
\end{table*}

\subsection{Step 2: Uniformity Enforces Content Injectivity}

From Step 1, we know the probe has collapsed the style dimension. However, a trivial solution $g^*(h) = \text{constant}$ (mapping all intents to a single point) also satisfies $\mathcal{L}_{align} = 0$. This is where the Uniformity Loss comes in.

\begin{lemma}[Content Injectivity]
\label{lemma:injectivity}
Minimizing $\mathcal{L}_{unif}$ ensures that the function $\psi(c)$ is injective (invertible on its image).
\end{lemma}

\begin{proof}
    The KoLeo uniformity loss (and asymptotically, the maximization of differential entropy) minimizes the interaction energy between points. It is theoretically minimized when the distribution of outputs $z = g(h)$ is the Uniform distribution on the feature manifold (e.g., the hypersphere).

Recall that $z = \psi(c)$. The entropy of the output variable $Z$ is related to the input $C$ by:
\begin{equation}
    H(Z) = H(\psi(C))
\end{equation}
From information theory, the entropy of a deterministic function of a random variable is maximized if and only if the function is \textbf{injective} (one-to-one).
\begin{itemize}
    \item If $\psi$ maps two distinct intents $c_1 \neq c_2$ to the same point $z$ (feature collapse), the entropy $H(Z)$ strictly decreases because information about the distinction between $c_1$ and $c_2$ is lost.
    \item To minimize $\mathcal{L}_{unif}$ (which is equivalent to maximizing entropy $H(Z)$), the optimal $\psi^*$ must preserve all distinctions present in the input distribution $P(c)$.
\end{itemize}

Therefore, for distinct intents $c_i \neq c_j$, we must have $\psi^*(c_i) \neq \psi^*(c_j)$. This implies $\psi^*$ is an invertible mapping.
\end{proof}

\subsection{Conclusion}

Combining Lemma \ref{lemma:invariance} and Lemma \ref{lemma:injectivity}:
1.  From Alignment: $g^*(f(c, s)) = \psi(c)$ (The probe ignores style).
2.  From Uniformity: $\psi(c)$ is invertible (The probe distinguishes distinct intents).

Thus, the learned representation $z = g^*(h)$ is isomorphic to the true latent intent $c$. We can recover the intent via $c = \psi^{-1}(z)$. This completes the proof that minimizing the proposed objective guarantees the identifiability of the latent intent.

\section{Detailed Elucidation of Causal Assumptions}
\label{app:causal_assumptions}

In Section 4, we invoked Theorem 1 to guarantee the identifiability of the latent intent variable $c$. This theorem relies on two structural assumptions regarding the data generation process: \textbf{Independence} and \textbf{Connectivity}. Here, we provide a detailed, intuitive explanation of why these assumptions are rigorously satisfied in our TSC-GRPO framework.

\subsection{The Independence Assumption ($c \perp s$)}
\label{app:independence}

\textbf{Intuition: Breaking Spurious Correlations.}
In the natural distribution of language (and in standard pre-training data), Content ($c$) and Style ($s$) are often highly correlated. For example, a polite prefix like ``Sure, here is'' ($s_{compliant}$) is statistically strongly correlated with helpful/safe content ($c_{safe}$). Conversely, refusals are correlated with harmful content. This statistical dependency is precisely why shallow alignment fails: the model learns to rely on the style $s$ as a shortcut to predict the safety of $c$.

The Independence Assumption requires us to break this correlation, rendering the style variable $s$ statistically uninformative about the content $c$. In other words, knowing the prefix is ``Sure'' should give the model zero information about whether the intent is harmful or safe.

\textbf{Why it is satisfied by design?}
In our framework, this assumption is not merely an observation; it is a structural property enforced by construction via our data augmentation strategy. As detailed in Stage 1, we explicitly construct a balanced $2 \times 2$ contrastive matrix:

By feeding the probe samples from all quadrants of Table \ref{tab:independence_matrix}, we force the conditional probability $P(Intent | Style)$ to approach the marginal probability $P(Intent)$.
\begin{equation}
    P(c_{harm} | s_{sure}) \approx P(c_{harm} | s_{raw}) \approx P(c_{harm})
\end{equation}
Because we artificially attach compliant prefixes to harmful queries (Malicious Compliance) and partial prefixes to safe queries, the probe cannot minimize the loss by relying on the prefix. It is mathematically compelled to look deeper for the invariant features of $c$, thus satisfying the independence assumption $c \perp s$.

\subsection{The Connectivity Assumption (Augmentation Graph)}
\label{app:connectivity}

\textbf{What is the Augmentation Graph?}
The Augmentation Graph $\mathcal{G_S}$ is a conceptual graph where:
\begin{itemize}
    \item \textbf{Nodes} represent different ``views'' or realizations of the same underlying intent (e.g., the raw query $x_{raw}$, the jailbroken query $x_{sure}$, the adversarial query $x_{adv}$).
    \item \textbf{Edges} represent the ``positive pair'' relationships. If our training objective pulls two views together (treating them as having the same intent), an edge exists between them.
\end{itemize}
The Connectivity Assumption states that this graph must be connected. This means there must be a path between any two styles $s_i$ and $s_j$. If the graph were disconnected (e.g., two isolated islands), the probe could learn two separate, unrelated representations for the same intent, failing to unify them into a single concept $c$.

\textbf{Why it is satisfied: The Star Topology.}
In the context of Contrastive Learning and our method, this assumption is trivially satisfied because our augmentation strategy follows a \textbf{Star Topology} (or Hub-and-Spoke model).

Consider a specific harmful intent $c$ (e.g., ``bomb making''). The ``Center Node'' (Hub) is the raw query $q$ (Type I). All other augmented views are generated directly from this hub:
\begin{itemize}
    \item Type II ($q \oplus \text{``Sure''}$) is a perturbation of $q$. $\to$ Edge $(x_{raw}, x_{sure})$.
    \item Type III ($q \oplus \delta_{adv}$) is a perturbation of $q$. $\to$ Edge $(x_{raw}, x_{adv})$.
    \item Type IV ($q \oplus y_{1:k}$) is an extension of $q$. $\to$ Edge $(x_{raw}, x_{partial})$.
\end{itemize}

Even if we do not explicitly pair Type II (Sure) and Type III (Adversarial) directly in a single triplet, they are linked via the Hub ($x_{raw}$).
\begin{equation}
    x_{sure} \leftrightarrow x_{raw} \leftrightarrow x_{adv}
\end{equation}
By transitivity, pulling $x_{sure}$ close to $x_{raw}$ and $x_{raw}$ close to $x_{adv}$ forces $x_{sure}$ and $x_{adv}$ to be close to each other.
Therefore, as long as all augmentations share the original query $q$ as their common source, the Augmentation Graph consists of a single connected component. There are no isolated subgraphs, guaranteeing that the probe learns a unified, globally consistent representation for the intent $c$.

\end{document}